\documentclass[a4paper,10pt,oneside]{amsart}

\usepackage[english]{babel}
\usepackage[T1]{fontenc}
\usepackage{amsmath,amssymb,amsthm,amsfonts}
\usepackage{mathtools}
\usepackage{latexsym}
\usepackage{mathrsfs}
\usepackage{enumerate}
\usepackage[all]{xy}
\usepackage{overpic}
\usepackage{booktabs}
\usepackage{caption}
\usepackage{url}
\usepackage{hyperref}
 
\theoremstyle{plain}
\newtheorem{lemma}{Lemma}[section]
\newtheorem{proposition}[lemma]{\textbf{Proposition}}
\newtheorem{theorem}[lemma]{\textbf{Theorem}}

\theoremstyle{definition}
\newtheorem{definition}[lemma]{\textbf{Definition}}
 
\newtheorem*{notation}{\textbf{Notation}}

\theoremstyle{remark}
\newtheorem{remark}[lemma]{Remark}

\newcommand{\R}{\mathbb{R}}
\newcommand{\C}{\mathbb{C}}

\newcommand{\p}{\mathbb{P}}
\newcommand{\SE}{\mathrm{SE}_3}
\newcommand{\SO}{\mathrm{SO}_3}
\newcommand{\BSC}{\mathrm{BSC}}
\newcommand{\SBSC}{\mathrm{SBSC}}
\newcommand{\inv}{\mathrm{inv}}
\newcommand{\spn}{\mathrm{span}}
\renewcommand\epsilon{\varepsilon}
\renewcommand\tilde{\widetilde}

\newcommand{\tth}{\thinspace}

\title{Mobile Icosapods}

\author{M. Gallet}
 
\address{
M. Gallet \\
Johann Radon Institute for Computational and Applied Mathematics (RICAM) \\
Austrian Academy of Sciences \\
Altenberger Stra\ss e 69 \\
4040 Linz, AT.}
 
\email{matteo.gallet@ricam.oeaw.ac.at}

\author{G. Nawratil}

\address{
G. Nawratil \\
Institute of Discrete Mathematics and Geometry\\
Vienna University of Technology \\
Wiedner Hauptstra\ss e 8-10/104 \\
1040 Vienna, AT.}

\email{nawratil@geometrie.tuwien.ac.at}

\author{J. Schicho}

\address{
J. Schicho \\
Research Institute for Symbolic Computation \\
Johannes Kepler University \\
Altenberger Stra\ss e 69 \\
4040 Linz, AT.}
 
\email{josef.schicho@risc.jku.at}

\author{J.M. Selig}

\address{
J.M. Selig \\
School of Engineering\\
London South Bank University\\
London SE1 0AA, U.K. 
}
 
\email{seligjm@lsbu.ac.uk}

\begin{document}



\begin{abstract}
  Pods are mechanical devices constituted of two rigid bodies, the base and 
the platform, connected by a number of other rigid bodies, called legs, that are 
anchored via spherical joints. It is possible to prove that the maximal number 
of legs of a mobile pod, when finite, is~$20$. In $1904$, Borel designed a 
technique to construct examples of such $20$-pods, but could not constrain the 
legs to have base and platform points with real coordinates. We show that 
Borel's construction yields all mobile $20$-pods, and that it 
is possible to construct examples with all real coordinates.
\end{abstract}

\maketitle

\section*{Introduction}

A multipod is a mechanical linkage consisting of two rigid bodies, called the 
{\em base} and the {\em platform}, and a number of rigid bodies, called {\em 
legs}, connecting them. Each leg is attached to base and platform with spherical 
joints (see Figure~\ref{fig0}), so platform points are constrained to 
lie on spheres --- the center of each sphere is then the base point connected to 
the respective leg. If the platform can move respecting the constraints imposed 
by the legs we say that the multipod is {\em mobile}. 

Note that multipods are also studied within {\it Rigidity Theory} as
so-called {\em body-bar frameworks} \cite{White1987}, as two rigid bodies 
(platform and base) are connected by multiple bars (legs). Mobile multipods 
correspond to flexible body-bar frameworks, whose study is of great practical
interest for e.g.\ protein folding \cite{Schulze2014}.

We can model the possible configurations of a multipod using direct isometries 
of~$\R^3$, by associating to every 
configuration the isometry mapping it to a fixed configuration. At this 
point one can consider a {\em motion}, namely a one-dimensional set of direct 
isometries, and try to construct a multipod moving according to the fixed 
motion. This approach has a long history; there are motions allowing multipods 
with infinitely many legs, but among those that allow only a finite number of 
legs, the maximal number is~$20$. This was proved by Schoenflies, see 
Remark~\ref{rem:schoenflies}. Borel proposed a construction for icosapods 
(namely multipods with $20$ legs) leading to linkages whose motion is 
line-symmetric, i.e.\ whose elements are involutions, namely rotations 
by~$180^{\circ}$ around a line (see Figure~\ref{fig0}).

\begin{figure}[!ht] 
\begin{center}
 \begin{overpic}
    [width=120mm]{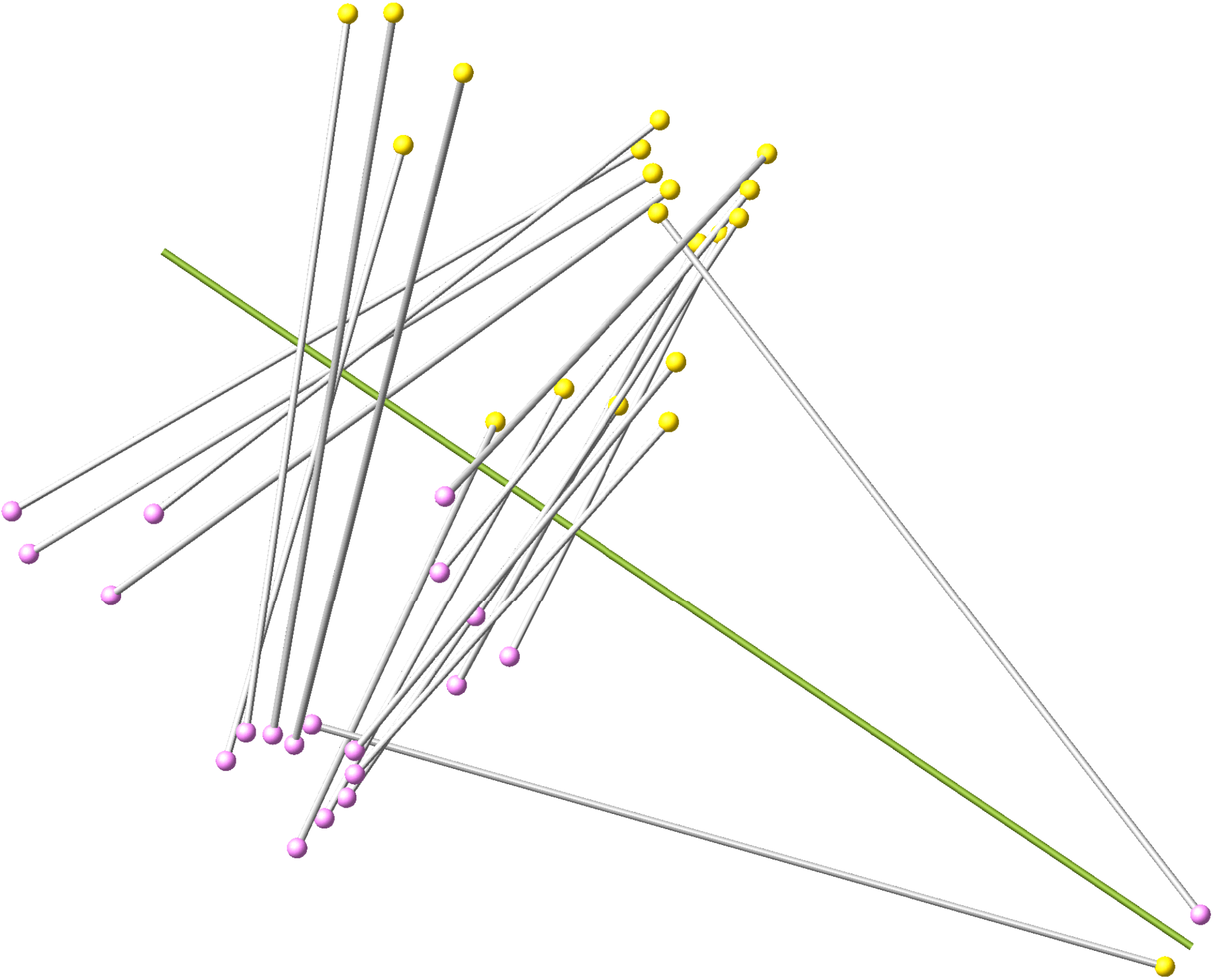}
\begin{small}
\put(11.5,60){$\ell$}
\end{small}
  \end{overpic} 
\end{center}
\caption{Base points (pink) and platform points (yellow) of an icosapod. 
Notice that for every leg there is a symmetric one obtained by rotating the 
first one by~$180^{\circ}$ around the axis~$\ell$. The symmetry reverses the 
role of base and platform points.
}
  \label{fig0}
\end{figure}

This paper provides two results on icosapods. First, we show that all mobile 
icosapods (with very mild restrictions) are instances of Borel's construction. 
Second, we exhibit a mobile icosapod that is an instance of Borel's 
construction (Borel, in fact, obtained equations for the base and platform 
points of a mobile icosapod, but could not prove the existence of solutions 
in~$\R^3$). The second part is closely related to the theory of quartic 
spectrahedra, which has been studied in~\cite{Ottem2015}.

Section~\ref{review} provides an historical overview of line-symmetric motions. 
In Section~\ref{duality} we set up the formalism and the objects that are 
needed for our approach to the problem; in particular we recall a 
compactification of the group of direct isometries that has already been used 
by the authors to deal with problems on multipods, and we show how the 
constraints imposed by legs can be interpreted as a duality between the space 
of legs and the space of direct isometries. In Section~\ref{completeness} we 
use these tools to prove that, under certain generality conditions, mobile 
icosapods admit line-symmetric motions. In Section~\ref{construction} we show 
how it is possible to construct example of mobile icosapods employing results 
in the theory of quartic spectrahedra.


\section{Review on line-symmetric motions}
\label{review}

Krames~\cite{kramesI} studied special one-parametric motions, obtained by 
reflecting the moving system~$\varsigma$ in the generators of a ruled surface in 
the fixed system~$\Sigma$. This ruled surface is called the {\it base surface} 
for the motion. He showed some remarkable properties of these motions 
(see~\cite{kramesI}), which led him to name them {\it Symmetrische Schrotung} 
(in German). This name was translated to English as {\it symmetric motion} by 
T\"olke~\cite{toelke}, {\it Krames motion} or {\it line-symmetric motion} by 
Bottema and Roth~\cite[page~319]{bottroth}. As each {\it Symmetrische Schrotung} 
has the additional property that it is equal to its inverse motion (cf.\ 
Krames~\cite[page~415]{kramesII}), it could also be called {\it involutory 
motion}. In this paper we use the name {\it line-symmetry} as it is probably the 
most commonly used term in today's kinematic community

Further characterizations of line-symmetric motions (beside the cited one of 
Krames~\cite{kramesI}) where given by T\"olke~\cite{toelke}, Bottema and 
Roth~\cite[Chapter~9, \S~7]{bottroth}, Selig and Husty~\cite{selighusty} and 
Hamann~\cite{marco}. If one uses the so-called Study parameters 
$(e_0:e_1:e_2:e_3:f_0:f_1:f_2:f_3)$ to describe isometries, then it is possible 
to characterize line-symmetric motions algebraically in the following 
way. Given such a motion, there always exist a Cartesian system of coordinates, 
or \emph{frame} $(o;x,y,z)$ for the moving 
system $\varsigma$ and a Cartesian frame $(O;X,Y,Z)$ for the fixed 
system~$\Sigma$ 
so that $e_0=f_0=0$ holds for the elements of the motion. With this choice 
of coordinates, the latter are rotations by~$180^{\circ}$ around lines; the 
Study coordinates $(e_1:e_2:e_3:f_1:f_2:f_3)$ of these isometries coincide with 
the Pl\"ucker coordinates of the lines.

\subsection{Historical results on line-symmetric motions with spherical 
paths}\label{review:hist}

In this paper we study line-symmetric motions that are solutions to the 
still unsolved problem posed by the French Academy of Science for the {\it Prix 
Vaillant} of the year~$1904$ (cf.~\cite{husty_bb}): {\it "Determine and study 
all displacements of a rigid body in which distinct points of the body move on 
spherical paths."} Borel and Bricard were awarded the prize for their 
papers~\cite{Borel} and~\cite{Bricard} containing partial solutions, and 
therefore this is also known as the {\it Borel Bricard (BB) problem}. 

\subsubsection{Krames's results}

Krames~\cite{kramesII,kramesV,kramesVI} studied some special motions already 
known to Borel and Bricard in more detail and stated the following theorem 
 \cite[Satz 6]{kramesII}:

\begin{theorem}\label{krames}
For each line-symmetric motion, that contains discrete, $1$ or 
$2$-dimensional spherical paths, the set~$\mathfrak{f}$ of points 
with spherical trajectories is congruent (direct isometry) to the 
set~$\mathfrak{F}$ of corresponding sphere centers. 
\end{theorem}
 
Moreover Krames noted in~\cite[page 409]{kramesII} that\footnote{The 
following extract as well as Theorem~\ref{krames} has been translated from 
the original German by the authors.} {\it "\ldots most 
of the solutions given by Borel and Bricard are line-symmetric motions. 
In each of these motions both geometers detected this circumstance 
by other means, without using the above mentioned result" 
(Theorem~\ref{krames}).} In the following we will take a closer look 
at the papers~\cite{Borel,Bricard}, which shows that the latter 
statement is not entirely correct.

\subsubsection{Bricard's results} 

Bricard studied these motions in~\cite[Chapitre VIII]{Bricard}. His 
first result in this context (see end of~\cite[\S~32, page~70]{Bricard}) 
reads as follows (adapted to our notation):

{\it Dans toutes les solutions auxquelles on sera conduit, les figures 
li\'{e}es $\mathfrak{F}$ et $\mathfrak{f}$ seront \'{e}videmment 
\'{e}gales et semblablement plac\'{e}es par rapport aux deux tri\`{e}dres 
$(O;X,Y,Z)$ et $(o;x,y,z)$.}

In the remainder of~\cite[Chapitre VIII]{Bricard} he discussed some 
special cases, which also yield remarkable results, but he did not give further 
information on the general case.

\subsubsection{Borel's results} 

Borel discussed in~\cite[Case~Fb]{Borel} exactly the case $e_0 = f_0 = 0$ 
and he proved in~\cite[Case~Fb1]{Borel} that in general a set of~$20$ 
points are located on spherical paths but without giving any result on 
the reality of the $20$ points. Moreover he studied two special cases 
in Fb2 and Fb3.

Borel did not mention the geometric meaning of the assumption 
\mbox{$e_0=f_0=0$}. He 
only stated at the beginning of case F~\cite[page~95]{Borel} that the moving frame 
$(o;x,y,z)$ is parallel to the frame 
obtained by a reflection of the fixed frame $(O;X,Y,Z)$ in a line. This 
corresponds 
to the weaker assumption $e_0=0$. He added that 
this implies the same consequences as already mentioned in~\cite[page~47, case 
C]{Borel}, which reads as follows (adapted to our notation): 

{\it \ldots dans le cas o\`{u} les tri\`{e}dres sont sym\'{e}triques par 
rapport \`{a} une droite, si deux courbes sont repr\'{e}sent\'{e}es par 
des \'{e}quations identiques, l'une en $X,Y,Z$, l'autre en 
$x,y,z$, elles sont sym\'{e}triques par 
rapport \`{a} cette droite.}

But Borel did not mention, neither in case Fb1 nor in his conclusion section, 
that $\mathfrak{f}$ with $\# \mathfrak{f} = 20$ is congruent 
to~$\mathfrak{F}$ (contrary to other special cases e.g.\ Fb3, where the 
congruence property is mentioned explicitly.)

\subsection{Review of line-symmetric self-motions of hexapods} 

We denote the platform points of the $i$-th leg in the moving 
system~$\varsigma$ by~$p_i$ and its corresponding base 
points in the fixed system~$\Sigma$ by~$P_i$.

For a generic choice of the geometry of the platform and the base as well as 
the leg lengths~$d_i$ the hexapod can have up to~$40$ configurations. 
Under certain conditions it can also happen that the direct kinematic problem 
has no discrete solution set but an $n$-dimensional one with $n>0$. Clearly 
these so-called self-motions of hexapods are solutions to the BB problem. 

In practice hexapods appear in the form of Stewart-Gough manipulators, 
which are $6$ degrees of freedom parallel robots. In these machines the leg 
lengths can be actively changed by prismatic joints and all spherical joints 
are passive. 

Moreover a hexapod (resp.\ Stewart-Gough manipulator) is called \emph{planar} if 
the points $p_1, \ldots, p_6$ are coplanar and also the points 
$P_1, \ldots, P_6$ are coplanar; otherwise it is called \emph{non-planar}. In 
the following we review those papers where line-symmetric self-motions of 
hexapods are reported.

\subsubsection{Non-planar hexapods with line-symmetric self-motions}

Line-symmetric motions with spherical paths already known to Borel~\cite{Borel} 
and Bricard~\cite{Bricard} (and also discussed by Krames 
in~\cite{kramesII,kramesVI}) were used by Husty and Zsombor-Murray~\cite{hzm} 
and Hartmann~\cite{hartmann} to construct examples of (planar and non-planar) 
hexapods with line-symmetric self-motions.

Point-symmetric hexapods with congruent platform and base possessing 
line-symmetric self-motions were given in~\cite[Theorem 11]{nawratil_ps}. 
Further non-planar hexapods with line-symmetric self-motions can be 
constructed from overconstrained pentapods with a linear platform 
\cite{nawratil_linpent}.

\subsubsection{Planar hexapods with line-symmetric 
self-motions}\label{review:planar}

All self-motions of the original Stewart-Gough manipulator were classified by 
Karger and Husty~\cite{kargerhusty}. Amongst others they reported a self-motion 
with the property $e_0=0$ (see \cite[page~$208$, last paragraph]{kargerhusty}), 
{\it "which has the property that all points of a cubic curve lying in the plane 
$\ldots$ and six additional points out of this plane have spherical 
trajectories. This seems to be a new case of a BB motion, not known so far."} 
Based on this result, Karger~\cite{karger1,karger2} presented a procedure for 
computing further {\it "new self-motions of parallel manipulators"} of the type 
$e_0=0$, where the points of a planar cubic~$\mathfrak{c}$ have spherical paths.

Another approach was taken by Nawratil in his series of papers 
\cite{nawratil_types,nawratil_basic,nawratil_nec,nawratil_jg}, 
by determining the necessary and sufficient geometric conditions for 
the existence of a $2$-dimensional motion such that three points in the 
$xy$-plane of~$\varsigma$ move on three planes 
orthogonal to the $XY$-plane of~$\Sigma$ (3-fold Darboux condition) and 
two planes orthogonal to the $xy$-plane of~$\varsigma$ 
slide through two fixed points located in the 
$XY$-plane of~$\Sigma$ (2-fold Mannheim condition). It turned out that 
all these so-called type II Darboux-Mannheim motions are line-symmetric. 
Moreover a geometric construction of a 12-parametric set of planar 
Stewart-Gough platforms (cf.~\cite[Corollary~5.4]{nawratil_jg}) with 
line-symmetric self-motions was given. It was also shown that the algorithm 
proposed by Karger in~\cite{karger1,karger2} yields these solutions. 

While studying the classic papers of Borel and Bricard for this historical 
review we noticed that the solution set of the BB problem mentioned in the last 
two paragraphs was already known to these two French geometers; 
cf.~\cite[Case~Fb3]{Borel} and~\cite[Chapter~V]{Bricard} (already reported by 
Bricard in~\cite[page~21]{Bricard_extra}). But in contrast to the above listed 
approaches (of Karger and Nawratil) both of them assumed that the motion with 
spherical trajectories is line-symmetric. Each of them additionally discovered 
one more property:
\begin{enumerate}[$\bullet$]
\item
Borel pointed out that there exist a further $8$ points (all~$8$ can be real) 
with spherical trajectories. This set of points splits in four pairs, 
which are symmetric with respect to the carrier plane of the 
cubic~$\mathfrak{c}$. 
\item
Bricard showed the following: If we identify the congruent planar cubics 
of the platform and the base, i.e.\ $\mathfrak{c}=\mathfrak{C}$, then the 
tangents in 
a corresponding point pair~$P$ and~$p$ with respect to 
$\mathfrak{c}=\mathfrak{C}$ intersect each other in a point of the cubic 
$\mathfrak{c}=\mathfrak{C}$ 
($P$ and $p$ form a so-called Steinerian couple).
\end{enumerate}

Bricard communicated his result (published in~\cite[page~$21$]{Bricard_extra}) 
to Duporcq, who gave an alternative reasoning in~\cite{duporcq}, which sank into 
oblivion over the past $100$ years. Only a footnote in the conclusion section of 
Borel's work~\cite{Borel} points to Duporcq's proof (but not to the original 
work of Bricard~\cite{Bricard_extra}), which is based on the following 
remarkable motion (see Figure~\ref{fig:complete_quadrilaterals}):

{\it Let $P_1, \ldots, P_6$ and $p_1, \ldots, p_6$ be the 
vertices of two complete quadrilaterals, which are congruent. Moreover the 
vertices are labelled in a way that $p_i$ is the opposite vertex of 
$P_i$ for $i \in \{1, \dotsc, 6\}$.
Then there exist a two-parametric line-symmetric motion where 
each~$p_i$ moves on a sphere centered at~$P_i$.}

\begin{figure}[top] 
\begin{center}
  \begin{overpic}[width=50mm]{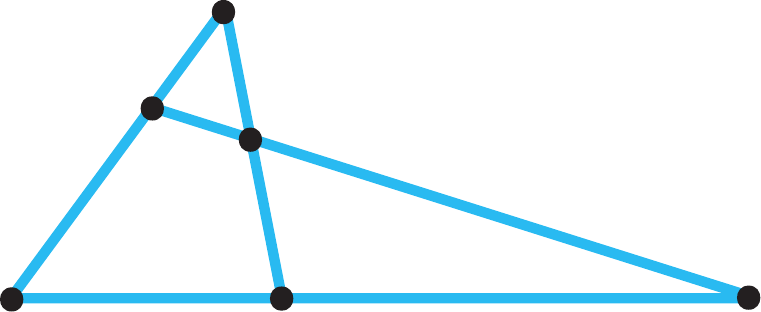}
    \begin{small}
      \put(7,5){$P_1$}
      \put(39,5){$P_2$}
      \put(96,6){$P_3$}
      \put(11.5,27){$P_5$}
      \put(35,25){$P_4$}
      \put(32,38){$P_6$}
    \end{small}
  \end{overpic} 
  \qquad
  \begin{overpic}[width=50mm]{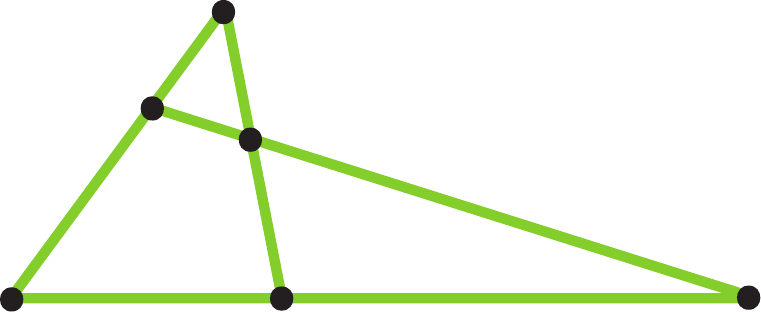}
    \begin{small}
      \put(7,5){$p_4$}
      \put(39,5){$p_5$}
      \put(96,6){$p_6$}
      \put(11.5,27){$p_2$}
      \put(35,25){$p_1$}
      \put(32,38){$p_3$}
    \end{small}
  \end{overpic} 
\end{center}
\caption{Illustration of Duporcq's complete quadrilaterals.
}
  \label{fig:complete_quadrilaterals}
\end{figure}
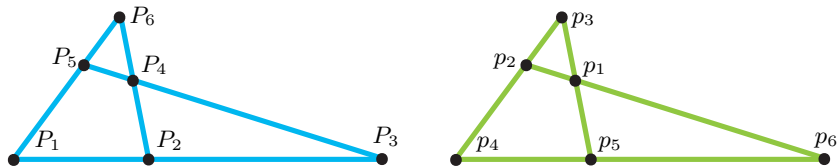

It can be easily checked that this configuration of base and platform points 
corresponds to an architecturally singular hexapod (e.g.~\cite{roschel} 
or~\cite{karger_arch}). As architecturally singular manipulators are redundant 
we can remove any leg --- w.l.o.g.\ we suppose that this is the sixth leg --- 
without changing the direct kinematics of the mechanism. Therefore the resulting 
pentapod $P_1, \ldots, P_5$ and $p_1, \ldots, p_5$ also has a 
two-parameter, line-symmetric, self-motion. 

Note that this pentapod yields a counter-example to Theorem~$4.2$ 
of~\cite{GalletNawratilSchicho2} and as a consequence also the 
work~\cite{asme_ns} is incomplete as it is based on this theorem. For the 
erratum to~\cite{GalletNawratilSchicho2} please see~\cite{erratum} and 
the addendum to~\cite{asme_ns} is given in~\cite{NawratilSchicho}.

\begin{remark}
If we assume for this pentapod that the line~$[P_1P_2]$ is the ideal 
line of the fixed plane (so $[p_4p_5]$ is 
the ideal line of the moving plane) then we get exactly the conditions 
found in~\cite{nawratil_types}, in which the points and ideal points 
of the plane's normals must fulfill in order to get a type II 
Darboux-Mannheim motion. Note that $P_2$, $P_3$ can also be complex conjugates 
(so $p_5$, $p_6$ are complex conjugate too). 
\end{remark}

\subsubsection{Computer search for mobile hexapods}
In \cite{schreyer}, Geiss and Schreyer describe a rather non-standard way
to find mobile hexapods. They set up an algebraic system of equations
equivalent to mobility, and then try random candidates with coordinates
in a finite field of small size with a computer. After collecting statistical
data indicating the existence of a family of real mobile hexapods, they
try a (computationally more expensive) lifting process in the most promising
cases. The method is extremely powerful and could still be used for finding
new families of mobile hexapods. However, the family reported in~\cite{schreyer}
can be seen as an instance of Borel's family Fb1. The line symmetry is not
apparent because the method starts by guessing 6 legs, which may not form
a line symmetric configuration; only if one adds the remaining 14 legs, or
at least all real legs among them, one obtains a symmetric configuration.
The question posed in Problem~$5$ in~\cite{schreyer} is easily answered from 
this viewpoint: two legs have the same lengths because they are conjugated by
line symmetry.

\section{Isometries, legs and bond theory}
\label{duality}

This section illustrates the concepts and techniques that will be needed to 
carry out our analysis. From now on, by an {\em $n$-pod} we mean a triple 
$(\vec{p}, \vec{P}, \vec{d})$ where $\vec{p}, \vec{P}$ are $n$-tuples of 
vectors in~$\R^3$ (respectively, platform and base points), and $\vec{d}$ is an 
$n$-tuple of positive real numbers (representing leg lengths). We describe the 
set of admissible configurations of a given pod~$\Pi$ as the set of all 
$\sigma \in \SE$ (where $\SE$ denotes the group of direct isometries from $\R^3$ 
to~$\R^3$) such that
\begin{equation*}
  \left\| \sigma(p_i) - P_i \right\| \; = \; d_i \quad 
\mathrm{for\ all\ } i \in \{ 1, \dotsc, n \}.
\end{equation*}
Multipods with $20$ legs are called \emph{icosapods}. 

\begin{remark}
  The description of admissible configurations that we adopt allows  
independent choices for two coordinate systems, one for the base and one for 
the platform: this is why also the leg lengths have to be included in the 
definition of a pod. Moreover, we do not require that the identity belongs 
to the admissible configurations, namely we do not fix any initial position of 
the pod.
\end{remark}

\subsection{Isometries and leg space}
\label{duality:isometries}

Throughout this paper we use a compactification of the group~$\SE$ of 
direct isometries that was introduced in~\cite{Selig2013} and 
in~\cite{GalletNawratilSchicho2014}. We briefly recall the construction. A 
direct isometry of~$\R^3$ can be described by a pair~$(M,y)$, where $M$ is an 
orthogonal matrix with~$\text{det}(M) = 1$ and $y$ is the image of the origin 
under the isometry. We define $x := -M^t y = -M^{-1} y$ and $r := \langle x,x 
\rangle = \langle y,y \rangle$, where $\langle \cdot, \cdot \rangle$ is the 
Euclidean scalar product. In this way, if we take coordinates $m_{11}, \dotsc, 
m_{33}$, $x_1, x_2, x_3$, $y_1, y_2, y_3$ and~$r$, together with a homogenizing 
variable~$h$, in~$\p^{16}_{{}_\C}$, then 
a direct isometry defines a point in projective space satisfying $h \neq 0$ and
\begin{equation}
\label{eq:compactification}
  \begin{gathered}
    M M^t \; = \; M^t M \; = \; h^2 \cdot \mathrm{id}_{\R^3}, \quad \det(M) 
    \; = \; h^3, \\
    M^t y + h x \; = \; 0, \quad M x + h y \; = \; 0 , \\
    \langle x,x \rangle \; = \; \langle y,y \rangle \; = \; r \tth h. 
  \end{gathered}
\end{equation}
Equations~\eqref{eq:compactification} define a variety~$X$ 
in~$\p^{16}_{{}_\C}$, whose real points satisfying $h \neq 0$ are in one to one 
correspondence with the elements of~$\SE$. A direct (computer aided) 
calculation shows that $X$ is a variety of dimension~$6$ and degree~$40$, and 
its saturated ideal is generated by quadrics. The main feature of this choice of 
coordinates is that, if $a, b \in \R^3$ and $d \geq 0$, then the condition 
$\left\| \sigma(a) - b \right\| = d$ on a direct isometry~$\sigma$ becomes 
linear, and in particular has the following form:
\begin{equation} 
\label{eq:sphere_condition}
  \bigl( \langle a, a \rangle + \langle b, b \rangle - d^2 \bigr) 
\tth h + r - 2 \tth \langle a, x \rangle - 2 \tth \langle y, b \rangle - 2 
\tth \langle M a, b \rangle \; = \; 0.
\end{equation}
We will usually refer to Equation~\eqref{eq:sphere_condition} as the 
\emph{sphere condition} imposed by $(a,b,d)$. Given an $n$-pod $\Pi = ( \vec{p}, 
\vec{P}, \vec{d} )$, we can consider its {\em configuration space}, namely the 
set of direct isometries~$\sigma$ in~$\SE$ such that $\left\| \sigma(a_i) - b_i 
\right\| = d_i$ for all $i \in \{1, \dotsc, n\}$. The fact that 
Equation~\eqref{eq:sphere_condition} is linear in the coordinates 
of~$\p^{16}_{{}_\C}$ means that the configuration space of~$\Pi$ can be 
compactified as the intersection of~$X$ with the linear space  
determined by the $n$ conditions imposed by its legs; we denote such variety 
by~$K_{\Pi}$. We say that the pod~$\Pi$ is {\em mobile} if the intersection 
$\SE \cap K_{\Pi}$ has (real) dimension greater than or equal to one. Notice 
that if $\Pi$ is mobile, then $K_{\Pi}$ has (complex) dimension greater than or 
equal to one.

\begin{notation}
  From now on, by the expression \emph{mobile icosapod} we mean a 
$20$-pod with mobility one to which we cannot add a leg without losing mobility.
\end{notation}

For our purposes, it is useful to introduce another $16$-dimensional projective 
space, playing the role of a dual space of the one where $X$ lives, where the 
duality is given by a bilinear version of Equation~\eqref{eq:sphere_condition}. 
We start by introducing a new quantity, called \emph{corrected leg 
length}, defined as $l := \langle a, a \rangle + \langle b, b \rangle - d^2$, 
so that Equation~\eqref{eq:sphere_condition} becomes
\[
  l \tth h + r - 2 \tth \langle a, x \rangle - 2 \tth \langle y, b \rangle - 2 
\tth \langle M a, b \rangle \; = \; 0.
\]
We think of the points $a = (a_1, a_2, a_3)$ and $b = (b_1, b_2, b_3)$ 
as points in~$\p^3_{{}_{\C}}$ by introducing two extra 
homogenization coordinates~$a_0$ and~$b_0$. In this way, the pair $(a,b)$ can 
be considered as a point in the Segre variety~$\Sigma_{3,3} \cong 
\p^3_{{}_\C} \times \p^3_{{}_\C}$; the latter is embedded 
in~$\p^{15}_{{}_{\C}}$, where we take coordinates~$\{ z_{ij} \}$ so that the 
points of~$\Sigma_{3,3}$ satisfy $z_{ij} = a_i \tth b_j$ for some 
$(a_0:a_1:a_2:a_3), (b_0:b_1:b_2:b_3) \in \p^3_{{}_{\C}}$. If we homogenize 
Equation~\eqref{eq:sphere_condition} with respect to the coordinates~$\{ z_{ij} 
\}$ and~$l$, then we get
\begin{equation}
\begin{multlined}[c][.9\displaywidth]
\label{eq:bilinear_sphere_condition}
  l \tth h + z_{00} r - 2 \tth ( z_{10} x_1 + z_{20} x_2 + z_{30} x_3 ) - \\ 
- 2 \tth (z_{01} y_1 + z_{02} y_2 + z_{03} y_3 ) - 2 \sum_{i, j = 1}^{3} 
m_{ij} \tth z_{ij} \; = \; 0.
\end{multlined}
\end{equation}
Notice that the left hand side of 
Equation~\eqref{eq:bilinear_sphere_condition} is a bilinear expression 
in the coordinates $(h, M, x, y, r)$ and in the coordinates~$(z, 
l)$. We denote this expression by~$\BSC$ (for \emph{bilinear sphere 
condition}). Hence, if we denote by~$\check{\p}^{16}_{{}_{\C}}$ the projective 
space with coordinates $(z,l)$, then we obtain a duality 
between~$\p^{16}_{{}_{\C}}$ and~$\check{\p}^{16}_{{}_{\C}}$ sending a point 
$(h_0, M_0, x_0, y_0, r_0) \in \p^{16}_{{}_{\C}}$ to the hyperplane 
in~$\check{\p}^{16}_{{}_{\C}}$ of equation $\BSC(h_0, M_0, x_0, y_0, r_0, z, 
l) = 0$, and a point $(z_0, l_0) \in \check{\p}^{16}_{{}_{\C}}$ to the 
hyperplane in~$\p^{16}_{{}_{\C}}$ of equation $\BSC(h, M, x, y, r, z_0, l_0) 
= 0$. 

\begin{remark}
\label{remark:cone_sphere_condition}
Suppose that $(z_0, l_0) \in \check{\p}^{16}_{{}_{\C}}$ belongs to the 
cone~$Y$ with vertex $(0: \dotsb : 0: 1)$ over the Segre variety~$\Sigma_{3,3}$ 
--- namely $(z_0)_{ij} = a_i \tth b_j$ for some $(a_0:a_1:a_2:a_3)$ and 
$(b_0:b_1:b_2:b_3)$. Suppose furthermore that $a_0 = b_0 = 1$ and $a_1^2 + 
a_2^2 + a_3^2 + b_1^2 + b_2^2 + b_3^2 - l_0 \geq 0$. Then from the definition 
of~$\BSC$ we see that the hyperplane $\BSC(h, M, x, y, r, z_0, l_0) = 0$ 
in~$\p^{16}_{{}_{\C}}$ is the same hyperplane defined by 
Equation~\eqref{eq:sphere_condition} where we take $a = (a_1, a_2, a_3)$, $b = 
(b_1, b_2, b_3)$ and $d = \sqrt{a_1^2 + a_2^2 + a_3^2 + b_1^2 + b_2^2 + b_3^2 - 
l}$.
\end{remark}

Remark~\ref{remark:cone_sphere_condition} indicates that the cone~$Y$ 
in~$\check{\p}^{16}_{{}_{\C}}$ plays a sort of dual role to the 
compactification~$X$ in~$\p^{16}_{{}_{\C}}$, and we will exploit this in our 
arguments. Since the Segre variety~$\Sigma_{3,3}$ has dimension~$6$ and 
degree~$20$, we obtain that $Y$ has dimension~$7$ and degree~$20$.

\begin{definition}
 Let $C \subseteq X$ be a curve. We define the {\em leg set} $L_C$ as the set of 
all points $(z,l) \in Y$ such that the~$\BSC$ --- evaluated at~$(z,l)$ --- 
holds for all points in~$C$. 
\end{definition}

\begin{remark}
 The leg set is a compactification of the set of all triples $(a,b,d) \in \R^3 
\times \R^3 \times \R_{\ge 0}$ such that the image of~$a$ under any point in~$C 
\cap \SE$ (considered as an isometry) has distance~$d$ from~$b$.
\end{remark}

\begin{proposition} 
\label{prop:leg_set_intersection}
Let $C \subseteq X$ be a curve. If $L_C$ has only finitely many complex points,
then its cardinality is at most~$20$. If $L_C$ has exactly $20$ points, then 
the linear span of~$L_C$ in~$\check{\p}^{16}_{{}_{\C}}$ is a projective subspace 
of dimension~$9$.
\end{proposition}

\begin{proof}
By construction $L_C$ is defined by linear equations as a subset of~$Y$; in 
other words $L_C = Y \cap \spn(L_C)$. Then the statement follows from 
general properties of linear sections of projective varieties. In fact, any 
linear subspace of codimension less than~$7$ intersects~$Y$ in a subvariety of 
positive dimension, hence $\dim \bigl( \spn(L_C) \bigr) \le 9$. A general 
linear subspace of dimension~$9$ intersects~$Y$ in $\deg(Y) = 20$ points.
In order to prove that $\dim \bigl( \spn(L_C) \bigr) = 9$ when $L_C$ has 
$20$ points we assume the contrary, that $\dim \bigl( \spn(L_C) \bigr) < 9$. Then 
we take a general linear superspace~$\Lambda$ of~$\spn(L_C)$ of dimension~$9$. 
It intersects~$Y$ again in~$20$ points, which must coincide with~$L_C$. On the 
other hand, if the Hilbert series of~$Y$ is~$\frac{P(t)}{(1-t)^{16}}$, then the 
Hilbert series of~$\Lambda \cap Y$ is~$\frac{P(t)}{(1-t)^{9}}$, but this 
contradicts the fact that $\Lambda \cap Y = L_C$ is contained in a linear space 
of dimension strictly smaller than~$9$.
\end{proof}

In order to show that the maximal number of~$20$ intersections, as described 
by Proposition~\ref{prop:leg_set_intersection}, can be achieved, we need to 
find a curve $C \subseteq X$ such that the bilinear sphere conditions of its 
points define a linear subspace in~$\check{\p}^{16}_{{}_{\C}}$ of dimension~$9$. 
This is equivalent to requiring $\dim \bigl(\spn(C) \bigr) = 15-9 = 6$. We will 
deal with this problem in Section~\ref{completeness}.

\begin{remark} \label{rem:schoenflies}
The proof of Proposition~\ref{prop:leg_set_intersection} gives also a simple 
alternative proof for the number of solutions (over~$\C$) of the spatial 
Burmester problem, which reads as follows: Given are seven poses 
$\varsigma_1,\ldots ,\varsigma_7$ of a moving 
system~$\varsigma$, determine all points of~$\varsigma$, that are 
located on a sphere in the given seven poses. 

The given poses correspond to seven points in~$X$, which span in the general 
case a~$\p^6_{{}_{\C}}$. Therefore its dual space is of dimension~$9$, and so 
intersects~$Y$ in exactly~$20$ points. 

This number was firstly computed by Schoenflies in~\cite[page~148]{schoenflies} 
and confirmed by Pimrose (see~\cite[footnote~3]{pimrose}). We conclude by 
observing that the first solution of the spatial Burmester problem using an 
approach based on polynomial systems was presented by 
Innocenti~\cite{innocenti}, who also gave an example with $20$ real solutions.
\end{remark} 
 
\subsection{Bond theory}
\label{duality:bond}

The second ingredient for our arguments is \emph{bond theory}, a technique for 
analyzing mobile multipods that has been introduced and developed (in the form 
we need here) in~\cite{GalletNawratilSchicho2014}. If $C$ is the configuration 
curve of a multipod~$\Pi$ of mobility one, then the \emph{bonds} of~$\Pi$ are 
defined as the intersections of~$C$ with the hyperplane $H = \bigl\{ h=0 
\bigr\}$. The set $B = H \cap X$, called the \emph{boundary}, has dimension~$5$ 
and is a real variety with only one real point; the latter can never occur as 
a bond, hence bonds always arise in complex conjugate pairs. The points 
of~$B$ do not represent isometries but they have geometric meaning: the bonds of 
a pod~$\Pi$ impose conditions on its legs. Five different types of boundary 
points can be distinguished:
\begin{description}
  \item[vertex] the only real point in~$B$; it is never contained in a 
configuration curve. 
  \item[inversion bonds] these correspond to inversion maps 
of orthogonal projections of base and platform; the~$\BSC$ of an inversion 
bond states that the projections of base point and platform point correspond 
via this inversion.
  \item[similarity bonds] these correspond to similarity maps of orthogonal 
projections of base and platform; the~$\BSC$ of an inversion 
bond states that the projections of base point and platform point correspond 
via this similarity.
  \item[butterfly bonds] these correspond to pairs of lines, one in the base 
and one in the platform; any leg satisfying the~$\BSC$ of a butterfly 
bond either has the base point on the base line or the platform point on the 
platform line.
  \item[collinearity bonds] these correspond to lines; if a multipod admits a 
colli- nearity bond, then either its platform points or its base points are 
collinear.
\end{description}

A mobile icosapod cannot have butterfly bonds (nor can it have collinearity 
bonds): this would imply that there are at least~$10$ legs with collinear base 
points or platform points, and then an infinite number of legs could be added 
without changing the mobility, which is disallowed by our definition of mobile 
icosapods.

The inversion bonds are smooth points of~$X$, and their tangent spaces are 
contained in~$H$. Consequently every motion passing through an inversion bond 
touches the hyperplane~$H$ tangentially at this bond. In particular, if a 
multipod~$\Pi$ of mobility one has no bonds other than the inversion bonds, the 
degree of its configuration curve~$C$ is twice the number of inversion bonds: 
the degree can be computed by intersecting with~$H$, and there we see only pairs 
of double intersections.

Consider the projection $\p^{16}_{{}_{\C}} \dashrightarrow \p^9_{{}_{\C}}$ 
keeping only the coordinates~$h$ and~$m_{ij}$ for $i,j \in \{1,2,3\}$. The open 
subset of the image of~$X$ defined by~$h \ne 0$ is the group variety~$\SO$, and 
the image itself is a subvariety~$X_{m}$ of degree~$8$ that is isomorphic to 
the Veronese embedding of~$\p^3_{{}_{\C}}$ by quadrics. 
The coordinates of~$\p^3_{{}_{\C}}$ are called \emph{Euler coordinates} and 
denoted by $e_0, e_1, e_2, e_3$. The center of the projection $\p^{16}_{{}_{\C}} 
\dashrightarrow \p^9_{{}_{\C}}$ intersects~$X$ in the union of the sets of 
similarity bonds, collinearity bonds and vertex. 

\subsection{The subvariety of involutions}
\label{duality:involutions}

We focus our attention on a particular subvariety of~$X$, the compactification 
of the set of involutions in~$\SE$. Involutions in~$\SE$ are rotations 
of~$180^{\circ}$ around a fixed axis, so their compactification --- which we 
will denote by~$X_{\inv}$ for reasons that will be clear later --- is a 
$4$-dimensional subvariety of~$X$, because the family of lines in~$\R^3$ is 
$4$-dimensional. One reason why involutions are particularly useful in the 
creation of mobile pods is that if $a$ and $b$ are a base and a platform point 
of a pod~$\Pi$, and $\sigma \in K_{\Pi}$ is an involution in the configuration 
space of~$\Pi$, then this means that $\left\| \sigma(a) - b \right\| = d$, where 
$d$ is the distance between $a$ and $b$; on the other hand, since $\sigma$ is an 
involution we have $\left\| \sigma(b) - a \right\| = d$. This means that if all 
isometries in the configuration space of~$\Pi$ are involutions, then we can swap 
the roles of base and platform points and obtain ``for free'' new legs not 
imposing any further restriction to the possible configurations of~$\Pi$.

If $\sigma = (M,y)$ is an involution, then $M = M^t$, that is to say $M$ is 
symmetric, and $y = x$. Hence we can consider the subvariety
\[
  \bigl\{ (h:M:x:y:r) \in X \, : \, M = M^t \text{ and } x = y \bigr\}.
\]
One verifies that such subvariety has two irreducible components, namely the 
isolated point corresponding to the identity and another one of dimension~$4$, 
which is cut out by a further linear equation, namely $m_{11} + m_{22} + 
m_{33} + h = 0$. 

\begin{definition}
  The subvariety of~$X$ defined by the equations $M = M^t$ and $x = y$ and 
$m_{11} + m_{22} + m_{33} + h = 0$ is denoted~$X_{\inv}$.
\end{definition}

\section{Mobile icosapods are line-symmetric}
\label{completeness}

In this section we will show that if an icosapod of mobility one admits 
an irreducible configuration curve, then its motion is line-symmetric 
(Theorem~\ref{thm:icosapod_line_symmetric}). We start by translating this 
concept in our formalism. Recall 
from~\cite[Section~2.2]{GalletNawratilSchicho2014} that the 
group~$\SE$ acts on its compactification~$X$: every isometry in~$\SE$ 
determines a projective automorphism of~$\p^{16}_{{}_{\C}}$ leaving~$X$ 
invariant. 

\begin{definition}
  Let $C \subseteq X$ be a curve. Then $C$ is called an {\em involutory motion} 
if $C \subseteq X_{\inv}$. The curve $C$ is called a {\em line-symmetric 
motion} if there exists an isometry~$\tau$ such that the automorphism 
associated to~$\tau$ maps~$C$ inside~$X_{\inv}$.
\end{definition}

From Proposition~\ref{prop:leg_set_intersection} we know that the configuration 
curve of an icosapod spans a linear subspace of dimension~$6$.
To get on overview of possible examples of irreducible curves~$C \subseteq X$ 
with $\dim \bigl( \spn(C) \bigr) = 6$, we consider the projection 
$\p^{16}_{{}_{\C}} \dashrightarrow \p^9_{{}_{\C}}$ described at the end of 
Subsection~\ref{duality:bond}. Let $C_m \subseteq X_{m}$ be the projection 
of~$C$, which can be either a point or a curve. It is possible to prove 
(see~\cite{Nawratil2014}) that if $C_m$ is a point, then there exists a multipod 
with infinitely many legs admitting $C$ as configuration set. Since we are 
interested in pods with finitely many legs, from now we suppose that  $C_m$ is a 
curve. Let $C_e \subseteq \p^3_{{}_{\C}}$ be its isomorphic preimage under the 
Veronese map. 

\begin{proposition} 
\label{proposition:planar_or_cubic}
If $C \subseteq X $ is an irreducible curve such that $\dim \bigl( \spn(C) 
\bigr) = 6$, then $C_e$ is either planar or a twisted cubic.
\end{proposition}
\begin{proof}
The projection of~$\spn(C)$ in~$\p^9_{{}_{\C}}$ is a linear subspace of 
dimension at most~$6$. Hence the ideal of~$C_m$ contains at least~$3$ linear 
independent linear forms. Hence the ideal of~$C_e$ contains at least~$3$ linear
independent quadratic forms. The proposition then follows.
\end{proof}

From now on, since we aim for a result on mobile icosapods, we will consider 
curves~$C$ allowing exactly $20$ legs, that is satisfying the following 
condition:
\begin{equation}
\label{eq:condition}
\left\{
\begin{array}{l}
  C \text{ is an irreducible real curve with real points,} \\
  L_{C} \text{ consists of exactly } 20 \text{ real finite points.}
\end{array}
\right.
\tag{$\dagger$}
\end{equation}
Here by real ``finite'' points we mean that their $z_{00}$-coordinates are not 
zero; in other terms, such points determine pairs of base and platform points 
in~$\R^3$ (and not at infinity).

\begin{remark}
Notice that condition~\eqref{eq:condition} does not comprise every mechanical 
device that one could name a ``mobile icosapod'': there may exist a device 
admitting infinitely many complex points in its leg space, of which only $20$ 
are real and finite. Still, we believe that condition~\eqref{eq:condition} is a 
good compromise because it will allow a uniform treatment of the topic. 
\end{remark}

\begin{remark}
\label{remark:consequences_dagger}
Notice that condition~\eqref{eq:condition} implies that $\dim \bigl( \spn(C) 
\bigr) = 6$. Moreover, from Section~\ref{duality:bond} it follows that $C$ does 
not pass through any butterfly or collinearity point.
\end{remark}

\begin{lemma}
\label{lemma:no_twisted}
  Suppose that $C \subseteq X$ satisfies condition~\eqref{eq:condition}. Then 
$C_e$ cannot be a cubic (neither a twisted cubic, nor a plane cubic).
\end{lemma}
\begin{proof}
Suppose that $C$ satisfies condition~\eqref{eq:condition} and $C_e$ is a twisted 
cubic. Then the curve~$C_m$, isomorphic to~$C_e$ under 
the Veronese embedding, is a rational normal sextic, and therefore spans a 
linear space of dimension~$6$. This means that the projection $C \longrightarrow 
C_m$ is a projective isomorphism, which means that the center of the 
projection is disjoint from~$\spn(C)$. By Section~\ref{duality:bond}, this 
implies that $C$ does not admit similarity or collinearity bonds. Recall from 
Remark~\ref{remark:consequences_dagger} that the curve~$C$ does not admit 
butterfly points, so the only ones left are the inversion points. However, 
since the degree of~$C$ is twice the number of inversion bonds, we would get $3$ 
of them, and this is not possible, since they come in conjugate pairs. 

Now, suppose that $C_e$ is a planar cubic. Then the curve~$C_m$ spans a linear 
space of dimension~$5$. This means that the center of the projection $C 
\longrightarrow C_m$ intersects~$\spn(C)$ in a single point. This shows again 
that $C$ does not admit similarity bonds, because the similarity bonds are 
contained in the center of projection and they occur pairwise. Hence we can 
conclude as in the above case.
\end{proof}

We focus therefore on curves~$C$ satisfying~\eqref{eq:condition} such that 
$C_e$ is planar of degree different from~$3$. First we rule out the case when 
$C_e$ is a line.

\begin{lemma} 
\label{lemma:planar_deg1}
Suppose that $C \subseteq X$ satisfies condition~\eqref{eq:condition}. Then 
$C_e$ cannot be a line.
\end{lemma}
\begin{proof}
 If $C_e$ is a line the corresponding motion can only be a Schoenflies 
motion\footnote{These motions can be composed by a rotation about a fixed axis 
and an arbitrary translation.}. These motions with points moving on spheres 
where previously studied by Husty and Karger in~\cite{HustyKarger}. It is not 
difficult to see that no discrete solution can exist, as any leg can be 
translated along the axis of the Schoenflies motion without restricting the 
self-motion. Therefore we always end up with an $\infty$-pod.
\end{proof}

\begin{lemma}
\label{lemma:planar_deg2}
Suppose that $C \subseteq X$ satisfies condition~\eqref{eq:condition}. Then 
$C_e$ cannot be a conic.
\end{lemma}
\begin{proof}
  Assume that $C_e$ is a conic. Consider the linear projection 
$\p^{16}_{{}_{\C}} \dashrightarrow \p^{9}_{{}_{\C}}$: under the bilinear 
spherical condition $\BSC$ it determines a subspace $\check{\p}^{9}_{{}_{\C}} 
\subseteq \check{\p}^{16}_{{}_{\C}}$. By a direct inspection of 
Equation~\eqref{eq:bilinear_sphere_condition} one notices that the subvariety 
$Y_{\infty} \subseteq Y$ composed of those pairs $(a,b)$ of points with $a_0 = 
b_0 = 0$ (namely legs for which both the base and the platform point are at 
infinity) is contained in $\check{\p}^{9}_{{}_{\C}}$. The dimension 
of~$Y_{\infty}$ is~$5$ and its degree is~$6$, since it is a cone over the 
Segre variety~$\p^2_{{}_{\C}} \times \p^2_{{}_{\C}}$.

Consider now the set of linear forms in~$\p^{9}_{{}_{\C}}$ vanishing on~$C_m$: 
this is a vector space of dimension~$5$, since it is isomorphic to the vector 
space of all quadratic forms on~$\p^3_{{}_{\C}}$ that vanish along~$C_e$. In 
this way we get a linear space $\check{\p}^{4}_{{}_{\C}} \subseteq 
\check{\p}^{9}_{{}_{\C}}$. The intersection $\check{\p}^{4}_{{}_{\C}} \cap 
Y_{\infty} \subseteq L_{C}$ is non-empty and is in general constituted of $6$ 
points. Thus the number of real legs not at infinity is at most~$14$, and this 
contradicts the assumption that $L_C$ has $20$ real finite points.
\end{proof}

For the remaining cases we want to prove that there exists an isometry~$\tau$ 
such that the image of~$C$ under the corresponding projective automorphism is 
contained in~$X_{\inv}$. Recall from Section~\ref{duality:bond} that we denoted 
$e_0, \dotsc, e_3$ the coordinates of the~$\p^3_{{}_{\C}}$ where $C_e$ lives. 
We may assume without loss of generality that $e_0 = 0$ holds for the points of 
$C_e$; this can be achieved by a suitable rotation of the coordinate frame of 
the platform --- specifically by acting on~$C$ with a suitable rotation --- since by 
assumption $C_e$ is planar. In terms of the coordinates of~$X$, this means that 
we can apply an automorphism of~$\p^{16}_{{}_{\C}}$ induced by an isometry so 
that the points of~$C$ satisfy $m_{ij} = m_{ji}$ and $m_{11} + m_{22} + m_{33} + 
h = 0$.

We can use the assumption $e_{0}=0$ in order to simplify the embedding. Let 
$X_1$ be the intersection of $X$ with the linear space
\[
  \bigl\{ (h:M:x:y:r) \in \p^{16}_{{}_{\C}} \, : \, m_{ij} = m_{ji} 
\text{ and } m_{11} + m_{22} + m_{33} + h = 0 \bigr\} .
\]
Notice that $X_1$ is the set-theoretical preimage of the locus $\{e_0 = 0\}$ 
under the projection $X \dashrightarrow \p^3_{{}_{\C}}$. We project~$X_1$ from 
the point $(0: \dotsb 0: 1)$ in~$\p^{16}_{{}_{\C}}$ --- the only 
singular point of~$X$ of order~$20$ --- giving a subvariety~$X_2$
of~$\p^{11}_{{}_{\C}}$. Consider the map from $X_2$ to the weighted 
projective space $\p_{{}_{\C}}(\vec{1}, \vec{2})$ of dimension~$5$ (here 
$\vec{1} = (1,1,1)$ and $\vec{2} = (2,2,2,2,2,2)$), with coordinates 
$e_1,e_2,e_3$ of weight~$1$ and coordinates $p_1,p_2,p_3,q_1,q_2,q_3$ of 
weight~$2$ obtained in the following way: express the coordinates~$h$ 
and~$m_{ij}$ for $i,j \in \{1,2,3\}$ in terms of the Euler coordinates $e_1, 
e_2, e_3$, and apply the coordinate change
\[ 
  p_i \; = \; x_i+y_i, \quad q_i \; = \; x_i-y_i, \qquad \text{for 
  } i \in \{1,2,3\}. 
\]
The image of~$X_2$ is a weighted projective variety $Z \subseteq 
\p_{{}_{\C}}(\vec{1}, \vec{2})$ of dimension~$5$ defined by the 
equations
\begin{multline*} 
  e_1 \tth p_1 + e_2 \tth p_2 + e_3 \tth q_3 \; = \; 
  p_1 \tth q_1 + p_2 \tth q_2 + p_3 \tth q_3 \; = \\ 
  e_1 \tth q_2 - e_2 \tth q_1 \; = \;
  e_1 \tth q_3 - e_3 \tth q_1 \; = \; 
  e_2 \tth q_3 - e_3 \tth q_2 \; = \; 0. 
\end{multline*}
Therefore, for a curve~$C \subseteq \p^{16}_{{}_{\C}}$ for which $C_e$ is 
planar and satisfies $e_0 = 0$ we get a map $C \longrightarrow C_z \subseteq Z$ 
that is the composition of a projection, a linear change of variables and a 
Veronese map.

\begin{remark}
 At the beginning of the section we pointed out that isometries determine 
automorphisms of~$\p^{16}_{{}_{\C}}$ leaving $X$ invariant. Notice that the 
automorphisms corresponding to translations leave~$X_1$ invariant, since its 
equations comprise only the rotational part of isometries. Therefore 
translations act also on~$Z$.
\end{remark}

\begin{lemma} 
\label{lemma:planar_deg4}
Suppose that $C \subseteq X$ satisfies condition~\eqref{eq:condition}. Let $C_z 
\subseteq Z$ be the image of the curve~$C$ under the previously defined maps. 
Then there exists a translation of the platform such that the corresponding 
automorphism maps $C_z$ to a curve $C'_z$ whose points 
satisfy $q_1 = q_2 = q_3 = 0$. 
\end{lemma}

\begin{proof}
  From the discussion so far we infer that $\deg(C_e) > 3$.
We are especially interested in the $q$-vector, so let $W 
\subseteq \p_{{}_{\C}}(1,1,1,2,2,2)$ be the projection of~$Z$ to the 
$e$ and $q$-coordinates and let $C_{w}$ be the image of~$C_z$ under such 
projection. The set~$W$ has dimension~$4$, and its equations are
\begin{equation}
\label{eq:equations_W}
  e_1 \tth q_2 - e_2 \tth q_1 \; = \; 
  e_1 \tth q_3 - e_3 \tth q_1 \; = \; 
  e_2 \tth q_3 - e_3 \tth q_2 \; = \; 0. 
\end{equation}
By a direct inspection of the map $C \longrightarrow C_{w}$ one notices that 
forms of weighted degree~$2$ on~$C_w$ correspond to linear form on~$C$. It 
follows that the vector space of weighted degree~$2$ forms on~$C_w$ has 
dimension at most~$7$. There are $9$ forms of weighted degree~$2$ on 
$\p_{{}_{\C}}(1,1,1,2,2,2)$ and they are all linear independent as forms on~$W$ 
 because the latter is defined by equations of weighted degree~$3$. Hence $C_w$ 
satisfies at least~$2$ equations~$E_1 = E_2 = 0$ of weighted degree~$2$. 

By construction, the polynomials $E_i$ are of the form $E_i = L_i 
\left(\vec{q}\right) + Q_i\left(\vec{e}\right)$, where $L_i$ is linear and 
$Q_i$ is quadratic. One notices that $L_1\left(\vec{q}\right)  
L_2\left(\vec{e}\right) - L_1\left(\vec{e}\right)  
L_2\left(\vec{q}\right)$ vanishes on~$W$, because it is a multiple of the 
polynomials in Equation~\eqref{eq:equations_W}. Therefore on~$C_w$ we have
\[
  E_1\left(\vec{e}, \vec{q}\right) \tth L_2\left( \vec{e} \right) - 
E_2\left(\vec{e}, \vec{q}\right) \tth L_1\left( \vec{e} \right) \; = \;
Q_1\left(\vec{e}\right) \tth L_2\left( \vec{e} \right) - 
Q_2\left(\vec{e}\right) \tth L_1\left( \vec{e} \right) \; = \; 0.
\]
The latter is a cubic equation only in the variables~$\vec{e}$, thus it is 
satisfied by~$C_e$. On the other hand, $C_e$ is a planar curve of degree greater 
than~$3$, so $C_e$ cannot satisfy a cubic nontrivial equation. Therefore we 
conclude that $Q_1\left(\vec{e}\right) \tth L_2\left( \vec{e} \right) - 
Q_2\left(\vec{e}\right) \tth L_1\left( \vec{e} \right)$ is zero 
on~$\p^{2}_{{}_{\C}}$. Since $L_1$ and $L_2$ cannot be proportional (otherwise 
we would be able to obtain from $E_1$ and $E_2$ a quadratic equation 
in~$\vec{e}$ satisfied by~$C_e$) we conclude by unique factorization that $L_2$ 
is a factor of~$Q_1$ and $L_1$ is a factor of~$Q_2$. Hence 
$Q_i\left(\vec{e}\right) = L\left(\vec{e}\right) \tth L_i\left(\vec{e}\right)$ 
for some linear polynomial~$L$. 

From Equation~\eqref{eq:equations_W} we infer that 
$L_1\left(\vec{q}\right) \tth e_j = L_1\left(\vec{e}\right) \tth q_j$ for 
$j \in \{1,2,3\}$. Since $E_1$ is zero on $C_{w}$, we have 
$-L_1\left(\vec{q}\right) = L\left(\vec{e}\right) \tth L_1\left(\vec{e}\right)$ 
on~$C_w$. Multiplying by $e_j$ the last equation yields:
\[
  -L_1 \left(\vec{e}\right) 
\tth q_j \; = \; -L_1 \left(\vec{q}\right) \tth e_j \; = \; 
L\left(\vec{e}\right) \tth L_1\left(\vec{e}\right) \tth e_j ,
\]
this implying that $q_j = L\left(\vec{e}\right) \tth e_j$ holds on~$C_w$ for 
$j \in \{1,2,3\}$.

One can verify that the automorphism corresponding to the translation by a 
vector $\vec{a} = (a_1,a_2,a_3) \in \R^3$ acts on the coordinates of 
$\p_{{}_{\C}}(1,1,1,2,2,2)$ by sending
\[
 \left( \vec{e}, \, \vec{q} \right) \; \mapsto \; \left( \vec{e}, \, \vec{q} + 
\ell_{\vec{a}} \tth \vec{e} \right),
\]
where $\ell_{\vec{a}} = a_1 \tth e_1 + a_2 \tth e_2 + a_3 \tth e_3$. Hence, if 
$L\left(\vec{e}\right) = \alpha_1 \tth e_1 + \alpha_2 \tth e_2 + \alpha_3 \tth 
e_3$, it is enough to apply to $C_{w}$ the automorphism corresponding to the 
translation by the vector $\alpha = (\alpha_1, \alpha_2, \alpha_3)$ to get that 
$q_1 = q_2 = q_3 = 0$ holds on~$C_w$. This proves the statement.
\end{proof}

We are now ready to prove our main result.
\begin{theorem}
\label{thm:icosapod_line_symmetric}
 Let $\Pi$ be a mobile icosapod such that its configuration 
curve~$K_{\Pi}$ is irreducible and satisfies \eqref{eq:condition}. Then 
$K_{\Pi}$ is a line-symmetric motion.
\end{theorem}
\begin{proof}
From the discussion so far (Proposition~\ref{proposition:planar_or_cubic}, 
Lemma~\ref{lemma:no_twisted} and the paragraph following it) we know that it is 
possible to apply a projective automorphism to~$K_{\Pi}$ so that the equations 
$m_{ij} = m_{ji}$ and $m_{11} + m_{22} + m_{33} + h = 0$ hold. Hence we only 
need to ensure $x = y$. However, in the new embedding in~$\p_{{}_{\C}}(\vec{1}, 
\vec{2})$ those equations correspond to $q_1 = q_2 = q_3 = 0$, so 
Lemma~\ref{lemma:planar_deg4}, Lemma~\ref{lemma:planar_deg2} and 
Lemma~\ref{lemma:planar_deg1} show the claim.
\end{proof}

\begin{remark}
From Theorem~\ref{thm:icosapod_line_symmetric} the set of base and platform 
points of the icosapod possesses a line-symmetry during the complete 
self-motion. But this property holds for any pod with a line-symmetric 
self-motion (see Theorem~$1.1$). As a consequence one could call these 
mechanical linkages "line-symmetric icosapods" by analogy to the "line-symmetric 
Bricard octahedra".
\end{remark}

\section{Construction of real icosapods}
\label{construction}

Borel proposed to construct line-symmetric icosapods simply by 
intersecting~$X_{\inv}$ with a general linear subspace~$T$ of dimension~$7$ in 
$\p^{10}_{{}_{\C}}$. Since $X_{\inv}$ is a variety of codimension~$6$ and 
degree~$12$ in~$\p^{10}_{{}_{\C}}$, the intersection $C = X_{\inv} \cap T$ is an 
irreducible curve of degree~$12$. Indeed, $C$ is a canonical curve of genus~$7$, 
as one can read off from the Hilbert series of~$X_{\inv}$. The projection~$C_e$ 
of~$C$ to the Euler parameters is a planar sextic. Recall that the leg set~$L_C$ 
is the intersection of~$Y$ with the dual of~$\spn (C)$, a linear subspace of 
codimension~$7$; this is, in general, a set of~$20$ complex points. It is not 
clear whether there are examples with~$20$ real legs, and the goal of this 
section is to show that there are instances of such curves~$C$ for which this is 
the case. We reduce the problem to a question on spectrahedra whose answer is 
well-known.

\subsection{Borel's construction}
We rephrase Borel's construction using the terminology and concepts 
introduced in this paper. Let $S$ be the linear subspace defined by 
the equations $M=M^t$ and $x=y$. Notice that $S$ has dimension~$10$ and 
contains~$X_{\inv}$, although $\spn(X_{\inv})$ is a linear subspace of 
dimension~$9$. The restriction of the bilinear sphere condition from 
Equation~\eqref{eq:bilinear_sphere_condition} can be written in a symmetric way 
as
\begin{equation}
\label{eq:symmetric_bilinear_sphere_condition}
  l \tth h + z_{00} r - 4\sum_{i = 1}^{3} s_{0i} x_i 
  -2\sum_{i = 1}^{3} z_{ii} m_{ii} -4\sum_{1\le i< j}^{3} s_{ij} m_{ij} \; = 
\; 0,
\end{equation}
where $s_{ij}=z_{ij}+z_{ji}$ for $1\le i<j\le 3$. We denote this equation
by $\SBSC$, for {\em symmetric bilinear sphere condition}. It defines a duality
between $S$ and a linear subspace $\check\p^{10}_{{}_{\C}} \subseteq 
\check\p^{16}_{{}_{\C}}$ whose projective coordinates are $l$,$z_{00}, \dotsc, 
z_{33}$,$s_{01},\dots,s_{23}$. The intersection of the leg variety~$Y$ with 
such $\check\p^{10}_{{}_{\C}}$ parametrizes pairs of legs obtained by swapping 
the roles of the base and the platform points. 
Denote by $\pi \colon \check\p^{10}_{{}_{\C}} \dashrightarrow 
\check\p^{9}_{{}_{\C}}$ the projection
defined by removing the $l$-coordinate. We denote by $Y_{\inv}$ the image
of the map $\alpha \colon \p^3_{{}_{\C}} \times \p^3_{{}_{\C}} \longrightarrow 
\check\p^{9}_{{}_{\C}}$,
\[
 \bigl( (a_0:\dots:a_3),(b_0:\dots:b_3) \bigr) 
\mapsto(\underbrace{a_0 \tth b_0}_{z_{00}} 
:\dotsb: \underbrace{a_3 \tth b_3}_{z_{33}}: \underbrace{a_0 
\tth b_1 + a_1 \tth b_0}_{s_{01}}: \dotsb: \underbrace{a_2 \tth b_3 + a_3 
\tth b_2}_{s_{23}}).
\]
One can easily prove that $Y_{\inv}$ is nothing but the projection under~$\pi$ 
of the intersection $Y \cap \check\p^{10}_{{}_{\C}}$. Note that $\alpha$ is a 
$2:1$ map, since $\alpha(a,b) = \alpha(b,a)$ for all $a,b \in 
\p^3_{{}_{\C}}$. Because of this, it might happen that two pairs of complex 
points are sent by~$\alpha$ to a real point of~$Y_{\inv}$. 

For any curve $C \subseteq X_{\inv}$, the leg set $L_C$ is equal to the 
intersection of the linear space~$\tilde{\Gamma}$, dual to~$\spn(C)$, with the 
cone over~$Y_{\inv}$ in~$\check\p^{10}$, namely $\pi^{-1} \left( Y_{\inv} 
\right)$. If $\dim \spn(C)=6$, then $\dim \tilde{\Gamma} = 3$. Since $X_{\inv}$ 
is contained in the hyperplane $\{ m_{11}+m_{22}+m_{33}+h=0 \}$, it follows that 
$\tilde{\Gamma}$ passes through the point~$p_e$ with coordinates $l=-2$, 
$z_{11}=z_{22}=z_{33}=1$ and all other coordinates being zero. Borel's 
construction can be rephrased as simply choosing a $3$-space passing 
through~$p_e$ and intersecting with the cone over $Y_{\inv}$. This cone has 
degree~$10$ and codimension~$3$ in $\check\p^{10}_{{}_{\C}}$, so 
generically there are $10$ solutions (possibly complex), each corresponding to 
a pair of legs.

For a general $3$-space $\tilde{\Gamma}$ passing through $p_e$, one can ask 
three questions on reality:
\begin{enumerate}
\item How many of the 10 points of $\tilde{\Gamma} \cap \pi^{-1} \left( 
Y_\inv \right)$ are real?
\item How many of the real points above have real preimages under~$\alpha$? 
Namely, how many real legs does the curve~$C$ admit?
\item Does the curve $X_{\inv} \cap \Lambda$ have real components, where 
$\Lambda$ is the dual to~$\tilde{\Gamma}$ under~$\SBSC$?
\end{enumerate}
The answers to Question~$(1)$ and~$(2)$ only depend on the projection 
of~$\tilde{\Gamma}$ to~$\check\p^9_{{}_{\C}}$.
In order to obtain positive answers for Question~$3$, it is also convenient to
start with the projection to~$\check\p^9_{{}_{\C}}$.

\begin{definition}
A {\em Borel subspace} $\Gamma$ is a 3-space in~$\check\p^9_{{}_{\C}}$ passing 
through~$\pi(p_e)$.
\end{definition}

The following proposition settles Question~$(3)$.

\begin{proposition}
Let $\Gamma$ be a Borel subspace. Then there exists a 3-space $\tilde{\Gamma}$ 
passing through~$p_e$ such that $\pi(\tilde{\Gamma}) = \Gamma$ and 
$X_{\inv} \cap \Lambda$ has real components, where $\Lambda$ is the dual 
of~$\tilde{\Gamma}$ under~$\SBSC$.
\end{proposition}

\begin{proof}
Let $f \colon S \dashrightarrow \p^4_{{}_{\C}}$ the projection from the linear 
subspace~$U$ dual to $\pi^{-1}(\Gamma)$. Then $U$ is 
contained in the hyperplane $\{ m_{11}+m_{22}+m_{33}+h=0 \}$. Hence the image 
of~$X_\inv$ under~$f$ is contained in a linear $3$-space, and $f|_{X_\inv}$ has 
one-dimensional fibers. Since $X_\inv$ has real components, it follows that 
there exist fibers $(f|_{X_\inv})^{-1}(q)$ with real components, for some 
$q \in \p^4_{{}_{\C}}$. We just need to choose~$\tilde{\Gamma}$ dual 
to~$f^{-1}(q)$; then $X_{\inv} \cap \Lambda$ coincides with 
$(f|_{X_\inv})^{-1}(q)$ and therefore has real components.
\end{proof}

In order to get some statistical data on the answers to Question~$(1)$ 
and~$(2)$, we tested $10000$ random examples\footnote{The Maple code used to 
perform such experiments can be downloaded from 
\url{http://matteogallet.altervista.org/main/papers/icosapods2015/Icosapods.mpl}
}. The results are 
shown in Tables~\ref{table:one} and~\ref{table:two}.

\begin{table}
\caption{Points in $\tilde{\Gamma} \cap \pi^{-1}(Y_\inv)$.\label{table:one}}
\begin{tabular}{rcccccc}
\toprule
\text{no. of real points} & 2 & 4 & 6 & 8 & 10 \\ 
frequency & 22 & 1067 & 3638 & 4035 & 1238 \\
\bottomrule
\end{tabular}
\end{table}

\begin{table}
\caption{Points in $\alpha^{-1}\bigl( \tilde{\Gamma} \cap \pi^{-1}(Y_\inv) 
\bigr)$.\label{table:two}}
\begin{tabular}{rccccccccccc}
\toprule
\text{no. of real points} & 2 & 4 & 6 & 8 & 10 & 12 & 14 & 16 & 18 & 
20 \\ 
frequency & 0 & 4107 & 0 & 5240 & 0 & 650 & 0 & 3 & 0 & 0 \\
\bottomrule
\end{tabular}
\end{table}

As one can see, the experimental data seem to indicate that there are no 
pods with $20$ real legs. This is, however, misleading; see the next 
section.

\subsection{Icosapods via spectrahedra}

We conclude our work by showing how it is possible to construct a mobile 
icosapod with $20$ real legs using some result in convex algebraic geometry.

Consider a 4-dimensional vector space $A$ of symmetric $4\times 4$-matrices 
over~$\R$. Classically, the {\em spectrahedron} defined by~$A$ is the 
subset of~$A$ comprised of positive semidefinite matrices. One can 
also consider the spectrahedron as a subset of the projective space 
$\p(A) \cong \p^3$. The boundary of the spectrahedron consists of the 
semidefinite matrices with determinant~$0$, and hence its Zariski closure is a 
quartic surface in~$\p^3$, called the {\em symmetroid} defined by~$A$. In 
general, a symmetroid has $10$ double points, corresponding to matrices of 
rank~$2$.

Given a spectrahedron whose symmetroid has $10$ complex double points, its 
{\em type} is the pair of integers $(a,b)$ where $a$ is the number of real 
double points of the symmetroid and $b$ is the number of real double points 
of the symmetroid that are also contained in spectrahedron.

\begin{theorem} 
\label{thm:spectrahedron}
There is a bijective correspondence between quartic spectrahedra containing 
the matrix $E := \left( \begin{smallmatrix} 
  0 & 0 & 0 & 0 \\
  0 & 1 & 0 & 0 \\
  0 & 0 & 1 & 0 \\
  0 & 0 & 0 & 1
\end{smallmatrix} \right)$
and Borel subspaces. For a spectrahedron defined by a vector space~$A$ and the 
corresponding Borel subspace~$\Gamma$, the following statement holds: 
if the spectrahedron has type $(a,b)$, then $\Gamma$ intersects~$Y_\inv$ 
in $a$ real points, and $a-b$ of them have real preimages under~$\alpha$.
\end{theorem}

\begin{proof}
We identify $\check\p^9_{{}_{\C}}$ with the projectivization of the vector 
space of symmetric $4\times 4$ matrices in the following way: a point with 
homogeneous coordinates $z_{00},\dots,s_{23}$ corresponds to the class of the 
matrix
\[
  \left( \begin{array}{cccc}
    2\tth z_{00} & s_{01} & s_{02} & s_{03} \\
    s_{01} & 2\tth z_{11} & s_{12} & s_{13} \\
    s_{02} & s_{12} & 2\tth z_{22} & s_{23} \\
    s_{03} & s_{13} & s_{23} & 2\tth z_{33}
  \end{array} \right).
\]
A linear subspace~$A$ of dimension~$3$ in the space of symmetric matrices 
containing the matrix~$E$ corresponds then to a Borel subspace~$\Gamma$.

The subvariety~$Y_\inv$ corresponds to the subvariety of matrices
of rank~$2$. A real matrix of rank~$2$ does not lie on the spectrahedron if and 
only if the quadratic form defined by it is a product of two distinct real 
linear forms, and this is true if and only if its preimage under~$\alpha$ 
is real. 
\end{proof}

Degtyarev and Itenberg in~\cite{Degtyarev2010} determined all possible types of
quartic spectrahedra. In particular, spectrahedra of type $(10,0)$ do exist, 
hence by Theorem~\ref{thm:spectrahedron} they provide Borel subspaces 
intersecting~$Y_{\inv}$ in $10$ real points, each of them having two real 
preimages under the map~$\alpha$. This implies that there exist Borel icosapods 
with $20$ real legs.

In~\cite{Ottem2015}, the authors give explicit examples of spectrahedra for all 
possible types.
The given example of type $(10,0)$ does not contain the matrix~$E$, but
it is easy to adapt their example to one of the same type that 
contains~$E$.

\subsection{Example}

Starting from \cite[Section~2, Case~(10, 0)]{Ottem2015} we 
computed\footnote{The Maple code containing a similar computation can be 
downloaded from 
\url{http://matteogallet.altervista.org/main/papers/icosapods2015/Icosapods.mpl}
} the following example, which is suitable for graphical 
representation:
\begin{align*}
P_1 = p_4 & = 
\left(-\tfrac{19493}{142100},-\tfrac{2088}{94325},-\tfrac{24}{9625} \right), 
& \quad 
p_1 = P_4 & = 
\left(-\tfrac{36411}{267844},-\tfrac{1608}{177793},\tfrac{504}{25399} \right), 
\\
P_2 = p_5 &=
\left(-\tfrac{269}{5000},\tfrac{39}{1000},\tfrac{17}{500} \right), &\quad
p_2 = P_5 & =
\left(-\tfrac{47}{368},-\tfrac{12}{1771},\tfrac{21}{1265} \right), \\
P_3 = p_6 &=
\left(-\tfrac{1863}{14645},-\tfrac{106851}{1555400},\tfrac{2509}{222200} 
\right), & \quad
p_3 = P_6 & =
\left(-\tfrac{15185}{112462},-\tfrac{120}{149303},\tfrac{48}{3047} \right).
\end{align*}

We apply a half-turn about a line $\ell$ through the point 
$(-\tfrac{1}{10},0,0)$ in direction 
$(1,\tfrac{7003716944}{10000000000},\tfrac{8}{10})$ to the platform. In the resulting initial 
position, which is illustrated in Figure~\ref{fig0}, the squared leg lengths 
of the first six legs read as follows:
\begin{equation*}
\begin{split}
d_1^2=d_4^2&=\tfrac{1081643179736912972309543483891375692}
{276669953748621822688942197018838171875}, \\
d_2^2=d_5^2&=\tfrac{219482305781081742844809989061}
{29002829339836395492656900000000},\\
d_3^2=d_6^2&=\tfrac{4185335506762812187908674782558830797}
{636621874987061375644008358435317156000}.
\end{split}
\end{equation*}
For this input data the self-motion consists of two components. The 
trajectories of the component which passes through the initial position are 
illustrated in Figures~\ref{fig2} and~\ref{fig3}. In the latter figure also 
the associated basic surface is displayed. An animation of this line-symmetric 
self-motion can be downloaded from 
\url{www.geometrie.tuwien.ac.at/nawratil/icosapod.gif}.

\begin{figure}[!ht] 
\begin{center}
 \begin{overpic}
    [width=120mm]{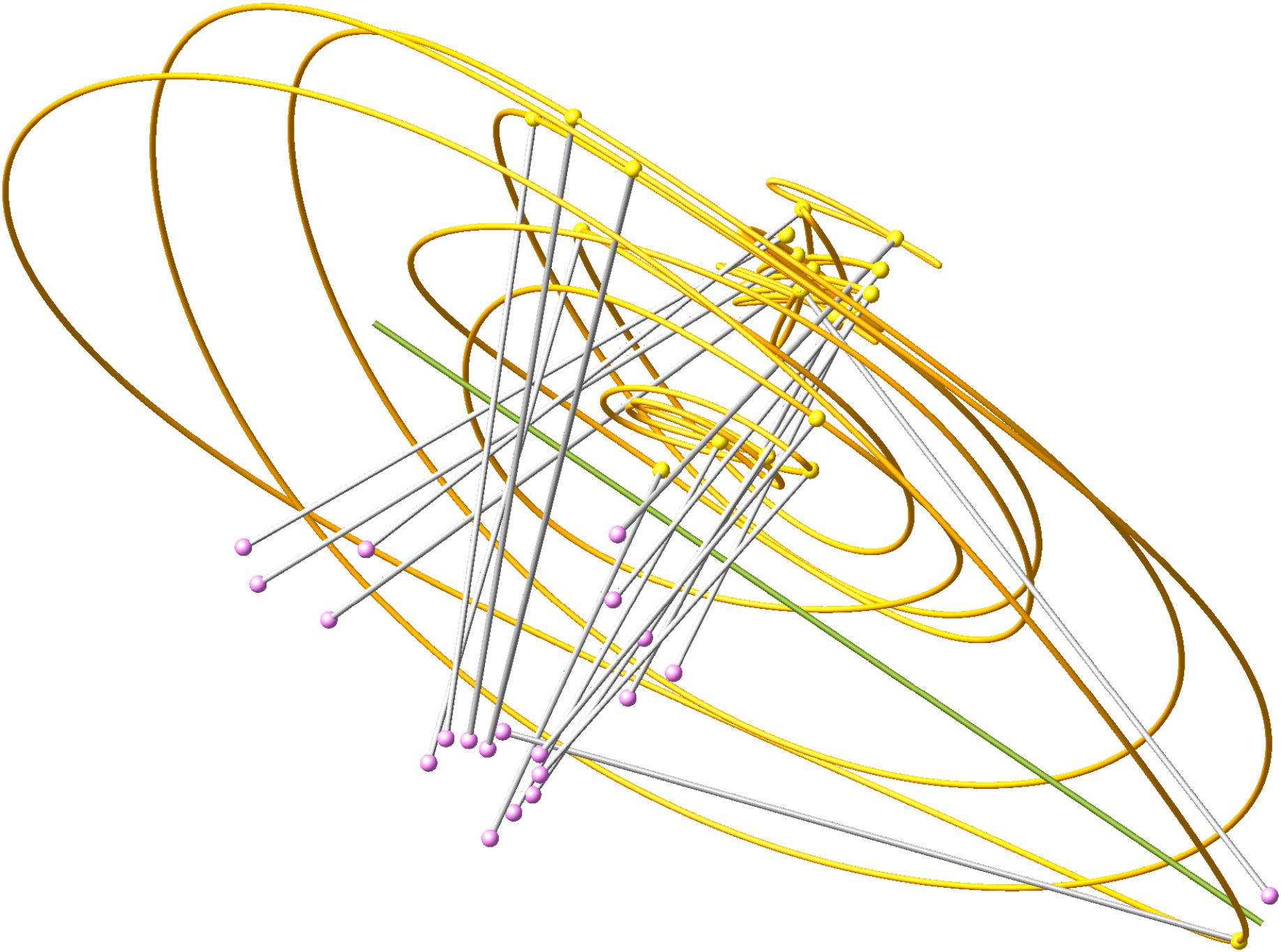}
\begin{small}
\put(29,49.9){$\ell$}
\end{small}
  \end{overpic} 
\end{center}
\caption{The $20$ spherical trajectories passing through the initial 
position of the icosapod.
}
  \label{fig2}
\end{figure}
   
\begin{figure}[!ht] 
\begin{center}
 \begin{overpic}
    [width=120mm]{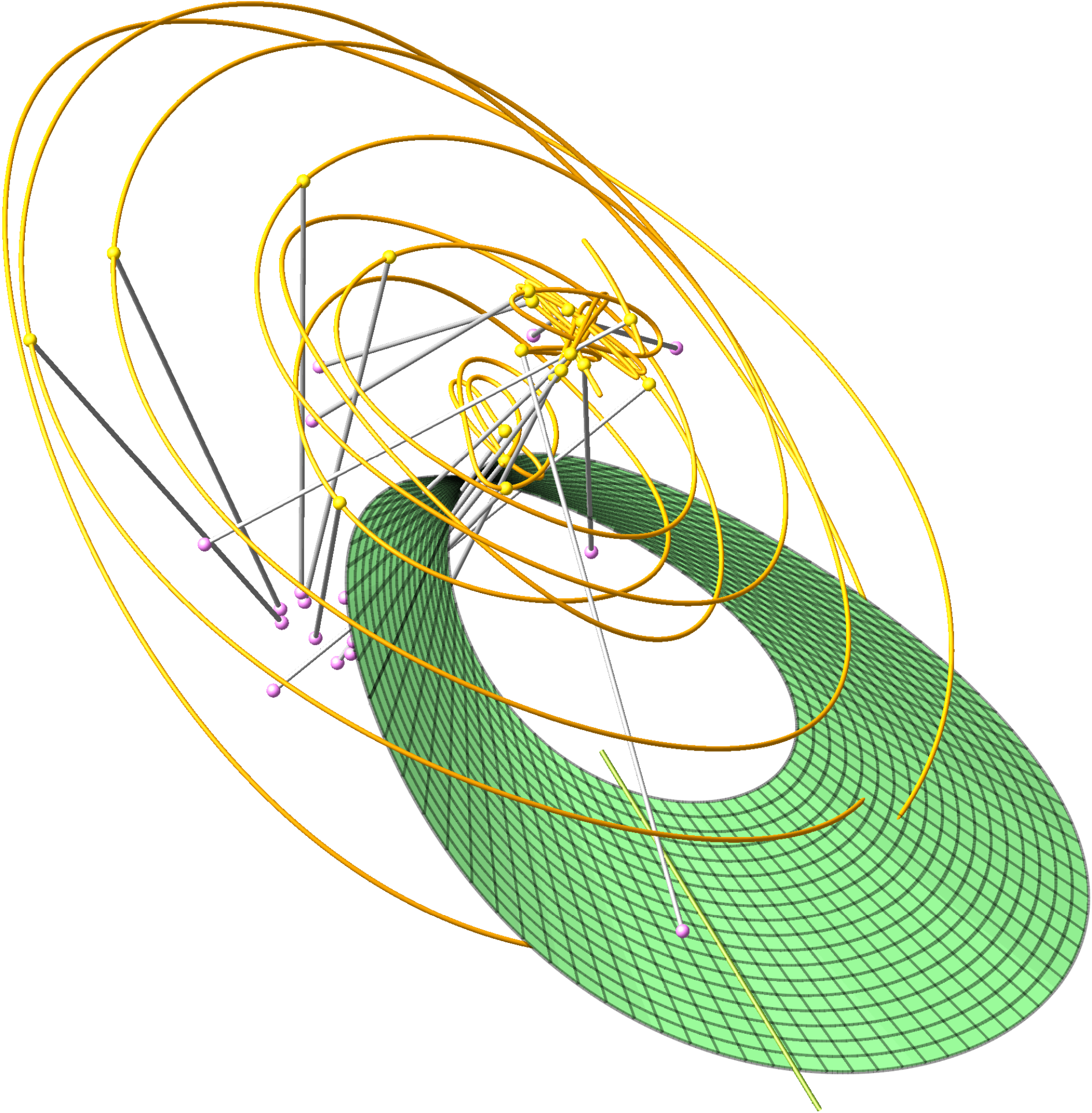}
\begin{small}
\put(70.2,0.2){$\ell$}
\end{small}
  \end{overpic} 
\end{center}
\caption{The same set-up as in Figure~\ref{fig2} but from another perspective. 
In addition a strip of the basic surface of this line-symmetric self-motion is 
illustrated, where $\ell$ is highlighted. Note that one of the two displayed 
families of curves on the surface is composed of straight lines, i.e.\ 
the generators of the ruled surface.
}
  \label{fig3}
\end{figure}

We close the paper by mentioning two open questions that we find of interest:
\begin{itemize}
\item[-]
Starting from spectrahedra of type $(a,b)$ with $a-b\ge 4$ one may construct 
mobile pods with $16$, $12$ or $8$ legs: is it true that a general mobile pod 
with $16$, $12$ or $8$ legs is line-symmetric? (It is known that $a,b$ have to 
be even numbers.)
\item[-]
Identify all cases 
where more than~$20$ points move on spheres during a line-symmetric 
motion; i.e.\ (a) 1-dim, (b) 2-dim or even (c) 3-dim set of points with spherical 
trajectories. Case~(c) is completely known due to Bricard~\cite{Bricard}, but 
cases~(a) and~(b) are still open. Examples for both cases are known (cf.\ 
Section~\ref{review:planar} and \cite{Borel, Bricard, kramesVI, 
nawratil_linpent}). 
\end{itemize}

\section*{Acknowledgments}

The first and third-named author's research is supported by the Austrian 
Science Fund (FWF): W1214-N15/DK9 and P26607 - ``Algebraic Methods in 
Kinematics: Motion Factorisation and Bond Theory''. The second-named author's 
research is funded by the Austrian Science Fund (FWF): P24927-N25 - 
``Stewart-Gough platforms with self-motions''. 

\bibliographystyle{amsalpha}
\bibliography{icosapods}

\end{document}